\documentclass{article} % For LaTeX2e
\usepackage[submission]{colm2026_conference}

\usepackage{microtype}
\usepackage{hyperref}
\usepackage{url}
\usepackage{booktabs}
\usepackage{amsthm}
\usepackage{amsthm}
\usepackage{enumitem}
\newtheorem{lemma}{Lemma}
\usepackage{lineno}

\definecolor{darkblue}{rgb}{0, 0, 0.5}
\hypersetup{colorlinks=true, citecolor=darkblue, linkcolor=darkblue, urlcolor=darkblue}

\usepackage{hyperref}
\usepackage{url}
\usepackage{booktabs}
\usepackage{amsmath,amssymb,amsfonts}
\usepackage{graphicx}
\usepackage{xcolor}
\usepackage{wrapfig}
\usepackage{xspace}
\usepackage{multirow}
\usepackage{cleveref}
\usepackage{algorithm}
\usepackage{algpseudocode}
\newtheorem{theorem}{Theorem}
\usepackage{lineno}
\usepackage{pifont}
\usepackage[most]{tcolorbox}

\newtcolorbox{observationbox}{
  colback=blue!4,
  colframe=blue!40!black,
  boxrule=0.5pt,
  arc=2pt,
  left=6pt, right=6pt, top=6pt, bottom=6pt,
  fonttitle=\bfseries,
  title=Key Observation
}

\definecolor{darkblue}{rgb}{0, 0, 0.5}
\hypersetup{colorlinks=true, citecolor=darkblue, linkcolor=darkblue, urlcolor=darkblue}

\title{Sampling for Quality: Training-Free Reward-Guided LLM Decoding via Sequential Monte Carlo}

\author{
Jelena Markovic-Voronov\thanks{Equal contribution.}\qquad
Wenhui Zhu\footnotemark[1]\qquad
Bo Long \\
Zhipeng Wang\thanks{Corresponding author: dagarwal@linkedin.com, zhipwang@linkedin.com}\qquad
Suyash Gupta \qquad
Kayhan Behdin \\
Bee-Chung Chen \qquad
Deepak Agarwal \footnotemark[2] \thanks{Project Leader}\\
\Large{LinkedIn}
}
\linenumbers

\date{}
\renewcommand{\headrulewidth}{0pt}

\begin{document}

\ifcolmsubmission
\linenumbers
\fi

\maketitle

\begin{abstract}

We introduce a principled probabilistic framework for reward-guided decoding in large language models, addressing the limitations of standard decoding methods that optimize token-level likelihood rather than sequence-level quality. Our method defines a reward-augmented target distribution over complete sequences by combining model transition probabilities with prefix-dependent reward potentials. \emph{Importantly, the approach is training-free: it leaves model weights unchanged and instead modifies the inference distribution via reward potentials, with all gains arising purely from inference-time sampling.} To sample from this distribution, we develop Sequential Monte Carlo algorithms, including a computationally efficient prefix-only variant and a lookahead variant whose intermediate targets match the exact marginals of the full sequence distribution. The framework also integrates resample--move updates with Metropolis--Hastings rejuvenation and supports block-wise generation, subsuming common decoding strategies such as temperature sampling and power-tempered objectives. Empirical results across three 7B models show significant gains. On code generation (HumanEval), our method improves base performance by up to \textbf{54.9\%} and surpasses the strongest sampling baselines by \textbf{9.1\%}--\textbf{15.3\%}. On mathematical reasoning (MATH500), it achieves gains of up to \textbf{8.8\%}. Notably, it reaches \textbf{87.8\%} on HumanEval and \textbf{78.4\%} on MATH500 with \textbf{Qwen2.5-7B}, consistently outperforming the reinforcement learning method GRPO.

\end{abstract}
% The abstract paragraph should be indented 1/2~inch (3~picas) on both left and
% right-hand margins. Use 10~point type, with a vertical spacing of 11~points.
% The word \textit{Abstract} must be centered and in point size 12. Two
% line spaces precede the abstract. The abstract must be limited to one
% paragraph.

\section{Introduction}

Large language models (LLMs) are typically used through autoregressive decoding,
where tokens are generated sequentially according to the model's next-token
distribution~\citep{vaswani2017attention, radford2018improving}. Common decoding strategies such as beam search
\citep{vijayakumar2016diverse}, nucleus sampling \citep{holtzman2020nucleus},
and temperature sampling \citep{wang2020contextual} are designed primarily to control diversity and
fluency of generated text. However, these approaches optimize token-level
likelihood rather than the quality of the full generated sequence.

In many important applications, the quality of a response depends on properties of the entire sequence rather than individual tokens. This is particularly
evident in reasoning, code generation, and mathematical problem solving,
where correctness or logical consistency can only be evaluated at the sequence
level. Several recent works therefore introduce additional scoring mechanisms,
such as verifiers for complete responses \citep{cobbe2021training}, process reward models
\citep{lightman2023step, uesato2022solving, zhang2025lessons}, or human feedback signals
\citep{ouyang2022training, bai2022constitutional}.
These signals provide \emph{reward potentials} that
score partial or complete outputs and guide generation toward desirable solutions.

Despite their widespread use, reward signals are typically incorporated only heuristically. Methods such as best-of-$N$ sampling \citep{brown2020language, huang2025best, chow2024bon} and re-ranking with process reward models \citep{setlur2024processverifiers, snell2024scaling} do not integrate rewards into the sampling distribution itself. As a result, generation is still largely driven by the base model likelihood, with rewards only applied after candidate generation. This can be suboptimal, particularly when reward signals are strong or misaligned with likelihood, and it prevents these methods from defining a coherent probabilistic target over full sequences.
% In an auxiliary analysis with an verifier, we find that \emph{weight-side} likelihood shaping (injecting a cumulative \(\alpha \sum_t \log p_\theta\) term into SMC resampling weights) can accelerate particle collapse when base model likelihood and verifier reward are misaligned, motivating reward-augmented targets and proposal design (More details refer to Appendix~\S\ref{sec:appendix:humaneval-misalignment}).

In this paper, we introduce a unified framework for sampling from
\emph{reward-modified sequence distributions} using Sequential Monte Carlo (SMC). We consider targets of the form
\begin{equation} \label{eq:unified-full-target}
\Pi(x_{1:T}\mid q)
\propto
\prod_{t=1}^{T} m_t(x_t\mid q,x_{<t})
\prod_{t=1}^{T} \psi_t(x_{1:t},q),
\end{equation}
where \(x_{1:T}\) is the generated sequence and \(q\) is the prompt. Here,
\(m_t\) is a transition factor derived from the base language model, and
\(\psi_t\) is a reward potential over prefixes. This defines a Feynman--Kac model, with the language model inducing the Markov process and the rewards acting as multiplicative potentials \citep{moral2004feynman,del2006sequential}. A similar form of a reward-modified base distribution appears in Direct Preference Optimization \citep{rafailov2023direct}, where rewards are incorporated at the whole sequence level.

Importantly, our method is \emph{training-free}: it leaves the base model weights unchanged and instead modifies the inference distribution through reward potentials. All improvements arise purely from inference-time sampling, without any additional training or fine-tuning.

Our framework unifies several decoding objectives, including
(i) \emph{tempered next-token targets} (temperature-based decoding) and
(ii) \emph{power-tempered sequence targets}, corresponding to posterior
sampling under a power prior \citep{ibrahim2000power, neal2001annealed},
which arise from different choices of \(m_t\).
Within this framework, we derive two types of intermediate sampling targets. The first is a \emph{prefix-only} target based on prefix factors of the model; it is computationally efficient but does not, in general, match the true marginals of the full sequence-level target. The second is a \emph{lookahead} target, which incorporates a correction for future tokens and yields intermediate distributions that exactly coincide with the marginals of the full target.
The framework also supports extensions such as block-wise generation of multiple tokens and Metropolis--Hastings (MH) rejuvenation within SMC, enabling efficient sampling while preserving the desired target distribution.
\setlength{\parskip}{0pt}
\paragraph{Contributions.}
The main contributions of this work are:
\begin{itemize}[leftmargin=15pt]

\item \textbf{Reward-aware probabilistic decoding.}
We introduce a probabilistic framework that directly incorporates reward potentials into the target distribution for LLM sequence generation.

\item \textbf{Unified formulation of decoding objectives.}
We show that temperature-based decoding and power-tempered sequence sampling emerge as special cases of a unified reward-modified target.

\item \textbf{SMC algorithms for reward-guided generation.}
We develop Sequential Monte Carlo algorithms for sampling from these targets,
including both prefix-only intermediate targets and lookahead targets that match the true marginals of the full sequence distribution.

\item \textbf{Empirical validation.}
We demonstrate significantly improved performance across multiple LLMs (Qwen2.5-7B,
Qwen2.5-Math-7B, DeepSeek-Math-7B) and benchmarks (MATH500, HumanEval, GPQA),
outperforming best-of-$N$, prior SMC-based approaches, and the reinforcement learning method GRPO.

\end{itemize}

% \item \textbf{Efficient extensions.}
% We develop block-wise sampling and Metropolis--Hastings rejuvenation moves that
% enable efficient generation while preserving the desired target distribution.

% Together, these results provide a principled Monte Carlo framework for
% reward-guided language generation and clarify the relationship between several
% recent sampling-based decoding methods for LLMs.
\section{Related Work}

Sequential Monte Carlo methods approximate sequences of distributions using weighted particle systems and were originally developed as particle filters \citep{gordon1993novel,doucet2001sequential}. 
They are naturally formulated in the Feynman--Kac framework, where a sequence of path-space distributions is defined by a Markov process and multiplicative potential functions \citep{moral2004feynman,del2006sequential}. 
The general SMC sampler of \citet{del2006sequential} extends these ideas to broader targets, with theoretical guarantees \citep{chopin2004central} and practical design principles such as resampling and proposal adaptation \citep{doucet2011tutorial}, making SMC a flexible tool for high-dimensional sequential inference.

Recent work has applied Monte Carlo methods to control language model generation. In particular, SMC-based approaches have been used to steer generation under structural or semantic constraints \citep{lew2023sequential,loula2025syntactic}, where \citet{lew2023sequential} samples from intermediate targets akin to our prefix-only formulation, and \citet{loula2025syntactic} incorporates a limited one-step lookahead. Additionally, adaptive weighted rejection sampling has been proposed for efficient controlled generation under local constraints \citep{lipkin2025fast}. Most closely related, \citet{zhao2024probabilistic} use twisted SMC with \emph{learned} twist functions, requiring additional neural networks to approximate target marginals, and focus on tasks such as toxic generation control, sentiment-conditioned review generation, and infilling. In contrast, our method is \emph{training-free}: our lookahead term corresponds to the optimal twist, which we estimate online and integrate with block-wise resample--move SMC and MH rejuvenation. \citet{park2024grammar} focus on grammar-aligned decoding and introduce an expected future grammaticality term, analogous to our lookahead, estimated via adaptive sampling.

Additionally, a line of research focuses on sampling from power-tempered variants of the base model distribution. Sampling from the power distribution $p(x_{1:T}\mid q)^{\alpha}$ without reward potentials has been explored using both Metropolis--Hastings and SMC methods. \citet{karan2025reasoning} introduce an MH algorithm for this setting, while \citet{azizi2026power} develop a faster SMC approach targeting the same distribution. In both cases, the intermediate targets are chosen as $p(x_{1:t}\mid q)^{\alpha}$, which do not correspond to the true marginals of the full sequence distribution. \citet{ji2026scalable} instead derive the exact next-token distribution implied by the power target and sample directly from it, avoiding MH or SMC, although relying on approximations to the target distribution. In contrast, our work incorporates reward potentials directly into the sampling target and develops SMC algorithms for sampling from the resulting reward-augmented sequence distribution.

\section{Method}

\subsection{Background}
\setlength{\parskip}{0pt}
\paragraph{Sequential Monte Carlo.}
Sequential Monte Carlo methods approximate a sequence of distributions using a set of weighted particles \citep{gordon1993novel,doucet2001sequential,del2006sequential}. 
Given unnormalized targets $\gamma_t(x_{1:t})$, SMC constructs particles $\{x_{1:t}^{(i)}\}_{i=1}^N$ together with associated importance weights $\{W_t^{(i)}\}_{i=1}^N$ iteratively as follows:
\begin{itemize}[leftmargin=15pt]
    \item Propagation: sample
    $x_t^{(i)} \sim r_t(\cdot \mid x_{<t}^{(i)})$ from a proposal distribution $r_t$.
    \item Weighting: define the incremental importance weight
    \begin{equation} \label{eq:smc-weights}
    w_t^{(i)}
    =
    \frac{\gamma_t(x_{1:t}^{(i)})}
         {\gamma_{t-1}(x_{1:t-1}^{(i)})\, r_t(x_t^{(i)} \mid x_{<t}^{(i)})},
    \end{equation}
    and update the (unnormalized) importance weights recursively as
    $W_t^{(i)} = W_{t-1}^{(i)} \, w_t^{(i)},\; W_0^{(i)} = 1.$
    \item Resampling (optional): particles are resampled according to their normalized weights $\widetilde W_t^{(i)} = W_t^{(i)} / \sum_{j=1}^N W_t^{(j)}$ to mitigate degeneracy, with resampling triggered when the effective sample size $\mathrm{ESS}_t = 1/\sum_{i=1}^N (\widetilde W_t^{(i)})^2$ falls below a prescribed threshold.
\end{itemize}

This framework is well suited to autoregressive models, where sequence distributions factorize over time, as in language model generation.
\setlength{\parskip}{0pt}
\paragraph{Resampling and MH Rejuvenation.}
Resampling reduces weight degeneracy but can decrease particle diversity. To address this, SMC is often combined with MCMC updates (``resample--move'' SMC), which preserve the target distribution while improving exploration. A common choice is a Metropolis--Hastings step \citep{gilks2001following}. Given a particle $x_{1:t}$, propose $x'_{1:t} \sim Q_t(\cdot \mid x_{1:t})$ and accept with probability
\begin{equation} \label{eq:mh-acceptance-prob}
a(x_{1:t}, x'_{1:t})
=
\min\left(
1,\;
\frac{\gamma_t(x'_{1:t}) Q_t(x_{1:t} \mid x'_{1:t})}
     {\gamma_t(x_{1:t}) Q_t(x'_{1:t} \mid x_{1:t})}
\right).
\end{equation}
These rejuvenation steps restore diversity and enable particles to better explore high-probability regions, which is particularly important for sharply peaked or reward-modified sequence distributions.

\subsection{Reward-Guided Tempered Base and Powered Base Distributions}

Our unified framework considers target distributions of the form in \eqref{eq:unified-full-target}. This framework covers two target distributions.
\setlength{\parskip}{0pt}
\paragraph{Target I: tempered next-token distribution (tempered base for short).}
Let
\begin{equation}
\tilde p_t^{(\alpha)}(x_t\mid q,x_{<t})
=
\frac{p(x_t\mid q,x_{<t})^{\alpha}}
{Z_t(x_{<t})},
\qquad
Z_t(x_{<t})=\sum_{v} p(v\mid q,x_{<t})^{\alpha}.
\end{equation}

which is the next-token distribution obtained from the base model with temperature $1/\alpha$.
Taking $m_t(x_t\mid q,x_{<t})=\tilde p_t^{(\alpha)}(x_t\mid q,x_{<t})$, \eqref{eq:unified-full-target} becomes
\begin{equation}
\label{eq:rho-full}
\Pi_{\mathrm{I}}(x_{1:T}\mid q)
\propto
\prod_{t=1}^{T}
\tilde p_t^{(\alpha)}(x_t\mid q,x_{<t})
\prod_{t=1}^{T}
\psi_t(x_{1:t},q).
\end{equation}

\setlength{\parskip}{0pt}
\paragraph{Target II: powered-tempered sequence distribution (powered base for short).}
Let
\begin{equation}
\label{eq:pi-full}
\Pi_{\mathrm{II}}(x_{1:T}\mid q)
\propto
p(x_{1:T}\mid q)^{\alpha}
\prod_{t=1}^{T}\psi_t(x_{1:t},q)=\prod_{t=1}^{T} p(x_t\mid q,x_{<t})^{\alpha}
\prod_{t=1}^{T}\psi_t(x_{1:t},q),
\end{equation}
where $p(x_{1:T}\mid q)=\prod_{t=1}^{T}p(x_t\mid q,x_{<t})$.
Thus this is also of the form \eqref{eq:unified-full-target}, with $m_t(x_t\mid q,x_{<t})=p(x_t\mid q,x_{<t})^{\alpha}.$

Target~I differs from Target~II by additional prefix-dependent normalization factors $Z_t(x_{<t})$, and therefore the two targets are not identical in general. The conditional next-token distributions $\Pi(x_t\mid q, x_{<t})$ for both targets are derived in Appendix \ref{sec:conditional:distributions}.
From a Bayesian perspective, the base language model can be viewed as a sequence prior $p(x_{1:T}\mid q)=\prod_{t=1}^T p_\theta(x_t\mid q,x_{<t})$, while the reward potential $\Psi(x_{1:T})=\prod_{t=1}^T \psi_t(x_{1:t},q)$ plays the role of a likelihood, so that $p(x_{1:T}\mid q)\Psi(x_{1:T})$ corresponds to the Bayesian posterior over sequences. Target~II corresponds to the \emph{power posterior} used in Bayesian inference, $\text{prior}^{\alpha}\cdot\text{likelihood}$, yielding $\Pi_{\mathrm{II}}(x_{1:T}\mid q)\propto p(x_{1:T}\mid q)^{\alpha}\cdot\Psi(x_{1:T})$. For $\alpha>1$, this distribution sharpens the prior and increases the probability of sequences that are already likely under the base model. \cite{karan2025reasoning} show that sampling from the power-tempered base distribution without incorporating a reward potential can achieve performance comparable to that obtained via reinforcement learning. In contrast, Target~I corresponds to temperature-based decoding and sharpens token-level probabilities $p(x_t\mid q,x_{<t})\rightarrow \tilde p_t^{(\alpha)}(x_t\mid q,x_{<t})$, thereby defining a modified autoregressive model.

\subsection{Prefix-Only and Lookahead Intermediate Targets for Sampling}

We now define two intermediate targets $\gamma_t(x_{1:t}\mid q)$ used in sampling with the same final distribution $\Pi(x_{1:T}\mid q)$. The first is prefix only, meaning its intermediate target at index $t$ depend on $m_s$ and $\psi_s$ for $s=1,\ldots, t$ and not the rest. The distribution of these intermediate targets do not correspond to the marginal distribution of $\Pi(x_{1:t}\mid q)$ since that distribution has dependence on $m_s$ and $\psi_s$ for $s>t$ as well. The second  is the lookahead intermediate target whose distribution correspond to $\Pi(x_{1:t}\mid q)$ exactly. 

\begin{lemma}[Prefix-only intermediate targets]
\label{lem:prefix_smc}
For either target, define the unnormalized prefix targets
\begin{equation}
\label{eq:generic-simple-prefix}
\gamma_t^{\mathrm{prf}}(x_{1:t}\mid q; \Pi)
=
\prod_{s=1}^{t} m_s(x_s\mid q,x_{<s})
\prod_{s=1}^{t} \psi_s(x_{1:s},q),
\qquad t=1,\dots,T.
\end{equation}
Then $\gamma_t^{\mathrm{prf}}$ satisfies the recursion
\begin{equation} \label{eq:prefix-only-gamma-recursion}
\gamma_t^{\mathrm{prf}}(x_{1:t}\mid q; \Pi)
=
\gamma_{t-1}^{\mathrm{prf}}(x_{1:t-1}\mid q; \Pi)\,
m_t(x_t\mid q,x_{<t})\,
\psi_t(x_{1:t},q).
\end{equation}

\end{lemma}

This lemma follows directly from the definitions of $\gamma_t^{\mathrm{prf}}$. The intermediate targets $\gamma_t^{\mathrm{prf}}$ and the corresponding SMC weights $w_t^{\mathrm{prf}}$ are instantiated for targets I and II in Appendix \ref{sec:instantiation}. It is worth noting that for proposal equal to tempered base $\tilde p^{(\alpha)}(x_t\mid q, x_{<t})$ the SMC weight updates $w_t^{\mathrm{prf}}(x_{1:t}; \Pi_{\mathrm{I}})$ simplify to only depend only on the reward potential and not base model probabilities.

%%%%%%%%%%%%%%%%%%%%%%%%%%%%%%%%%%%%%%%%%%%%%%%%%%%%%
%%%%%%%%%%%%%%%%%%%%%%%%%%%%%%%%%%%%%%%%%%%%%%%%%%%%%
%%%%%%%%%%%%%%%%%%%%%%%%%%%%%%%%%%%%%%%%%%%%%%%%%%%%%

\begin{lemma}[Lookahead intermediate targets]
Define the intermediate lookahead targets as the exact marginals
\begin{equation}
\label{eq:generic-lookahead-prefix-lemma}
\gamma_t^{\mathrm{look}}(x_{1:t}\mid q; \Pi)
=
\sum_{x_{t+1:T}}
\prod_{s=1}^{T} m_s(x_s\mid q,x_{<s})
\prod_{s=1}^{T} \psi_s(x_{1:s},q).
\end{equation}
Introduce the lookahead term
\begin{equation}
\label{eq:generic-lookahead-term-lemma}
L_t(x_{1:t},q; \Pi)
=
\sum_{x_{t+1:T}}
\prod_{s=t+1}^{T} m_s(x_s\mid q,x_{<s})
\prod_{s=t+1}^{T} \psi_s(x_{1:s},q).
\end{equation}
Then the prefix targets factorize as
\begin{equation}
\label{eq:generic-lookahead-factorization-lemma}
\gamma_t^{\mathrm{look}}(x_{1:t}\mid q; \Pi)
=
\left[
\prod_{s=1}^{t} m_s(x_s\mid q,x_{<s})
\prod_{s=1}^{t} \psi_s(x_{1:s},q)
\right]
L_t(x_{1:t},q; \Pi),
\end{equation}
and satisfy the recursion
\begin{equation}
\label{eq:generic-lookahead-recursion-lemma}
\gamma_t^{\mathrm{look}}(x_{1:t}\mid q; \Pi)
=
\gamma_{t-1}^{\mathrm{look}}(x_{1:t-1}\mid q; \Pi)\,
m_t(x_t\mid q,x_{<t})\,
\psi_t(x_{1:t},q)\,
\frac{L_t(x_{1:t},q; \Pi)}{L_{t-1}(x_{1:t-1},q; \Pi)}.
\end{equation}
\end{lemma}

The result follows directly from the definition of $\gamma_t^{\mathrm{look}}$. The specific forms of the SMC incremental weights $w_t^{\mathrm{look}}$ for Targets~I and~II are given in Appendix~\ref{sec:instantiation}. 

Our lookahead intermediate target corresponds to a twisted Feynman--Kac construction, where \(L_t\) acts as the optimal twisting (future-value) function, yielding the exact marginals of the full path measure~\citep{whiteley2012twisted}. This property is not shared by the prefix-only targets used in~\citet{azizi2026power} and~\citet{lew2023sequential}. Moreover, we estimate \(L_t\) online via Monte Carlo rollouts, making our method training-free. Consistent with this perspective, Theorem~\ref{theorem:mse} in Appendix~\ref{section:mse} shows that lookahead targets are optimal in that they minimize the mean squared error relative to the final importance weights. In contrast, prefix-only targets can incur large error when \(L_t\) deviates from one, while approximate lookahead targets based on an unbiased estimator \(\widehat{L}_t\) incur only vanishing excess error as the number of lookahead samples increases. We next construct this unbiased estimator $\widehat{L}_t$ separately for each target.

% Our lookahead intermediate target corresponds to a twisted Feynman--Kac construction, where \(L_t\) acts as the optimal twisting (future-value) function, yielding the exact marginals of the full path measure in line with \citet{whiteley2012twisted}. Consistent with this perspective, Theorem~\ref{theorem:mse} in Appendix~\ref{section:mse} shows that lookahead targets are optimal in that they minimize the mean squared error relative to the final importance weights. In contrast, prefix-only targets can incur large error when \(L_t\) deviates from one, while approximate lookahead targets based on an unbiased Monte Carlo estimator \(\widehat{L}_t\) incur only vanishing excess error as the number of lookahead samples increases.

\setlength{\parskip}{0pt}
\paragraph{Lookahead estimates for Target I.}
For the tempered next-token target, $m_t=\tilde p_t^{(\alpha)}$, so
\begin{equation} \label{eq:targetI:Lt}
L_t(x_{1:t},q; \Pi_{\mathrm{I}})
% \sum_{x_{t+1:T}}
% \prod_{s=t+1}^{T} \left(
% \tilde p_s^{(\alpha)}(x_s\mid q,x_{<s})
% \psi_s(x_{1:s},q) \right) 
= \mathbb{E}_{x_{t+1:T}\sim \tilde p^{(\alpha)}(\cdot\mid q,x_{1:t})}
\left[
\prod_{s=t+1}^{T}\psi_s(x_{1:s},q)
\right],
\end{equation}
where the expectation is under tempered conditional base distribution. Hence, if $x_{t+1:T}^{(1)},\dots,x_{t+1:T}^{(J)}
\sim
\tilde p^{(\alpha)}(\cdot\mid q,x_{1:t}),$
an unbiased estimator of $L_t$ is
\begin{equation}
\widehat L_t(x_{1:t}, q; \Pi_{\mathrm{I}})
=
\frac{1}{J}
\sum_{j=1}^{J}
\prod_{s=t+1}^{T}
\psi_s\!\big((x_{1:t}, x_{t+1:s}^{(j)}), q\big).
\end{equation}

\setlength{\parskip}{0pt}
\paragraph{Lookahead estimates for Target II.}
For the power-tempered sequence target, $m_t=p(x_t\mid q,x_{<t})^{\alpha}$, so
\begin{equation}\label{eq:targetII:Lt}
L_t(x_{1:t},q;\Pi_{\mathrm{II}})
% = \sum_{x_{t+1:T}} \prod_{s=t+1}^{T} p(x_s\mid q,x_{<s})^{\alpha}\psi_s(x_{1:s},q)
= \mathbb{E}_{x_{t+1:T}\sim p(\cdot\mid q,x_{1:t})}\!\left[p(x_{t+1:T}\mid q,x_{1:t})^{\alpha-1} \!\!\prod_{s=t+1}^{T}\psi_s(x_{1:s},q)\right],
\end{equation}
where the expectation is under the untempered conditional base distribution. Therefore, if
$x_{t+1:T}^{(1)},\dots,x_{t+1:T}^{(J)} \sim p(\cdot\mid q,x_{1:t}),$
a Monte Carlo estimator is
\begin{equation} \label{eq:lookahead-targetII-estimate}
\widehat L_t(x_{1:t},q;\Pi_{\mathrm{II}})
=
\frac{1}{J}
\sum_{j=1}^{J}
p\!\left(x_{t+1:T}^{(j)} \mid q, x_{1:t}\right)^{\alpha-1}
\prod_{s=t+1}^{T}
\psi_s\!\big((x_{1:t}, x_{t+1:s}^{(j)}), q\big).
\end{equation}

The block-wise construction of prefix-only and lookahead intermediate targets is a straightforward generalization of the token-wise derivations, so we defer the details to Appendix~\ref{section:blockwise-smc} and \ref{sec:blockwise-mh}, which cover the SMC and MH derivations, respectively.

%%%%%%%%%%%%%%%%%%%%%%%%%%%%%%%%%%%%%%%%%%%%%%%%
%%%%%%%%%%%%%%%%%%%%%%%%%%%%%%%%%%%%%%%%%%%%%%%%
%%%%%%%%%%%%%%%%%%%%%%%%%%%%%%%%%%%%%%%%%%%%%%%%
%%%%%%%%%%%%%%%%%%%%%%%%%%%%%%%%%%%%%%%%%%%%%%%%

\subsection{Algorithm}

We adopt a general resample--move SMC framework for sampling from reward-augmented sequence distributions. Our key design choice is to \emph{decouple} the intermediate target used for SMC importance weighting, $\gamma_t^{\mathrm{SMC}}$, from the target used in the MH rejuvenation step, $\gamma_t^{\mathrm{MH}}$. This separation enables an efficient division of labor: SMC focuses on fast exploration, while MH provides targeted correction. 

In the SMC generation stage, we use Target I with \emph{prefix-only} intermediate targets, $\gamma_t^{\mathrm{SMC}} = \gamma_t^{\mathrm{prf}}(\cdot \mid q; \Pi_{\mathrm{I}})$, together with the tempered base proposal $\tilde p^{(\alpha)}$. Under this choice, the incremental weights depend only on the reward potentials $\psi_t$, eliminating dependence on the base model likelihood and yielding a computationally lightweight particle generation procedure. This stage efficiently produces candidate trajectories, from which high-reward ones are retained.

To correct for degeneracy, we selectively refine duplicated low-reward particles via an MH rejuvenation step. Here, we switch to Target II with \emph{lookahead} intermediate targets, $\gamma_t^{\mathrm{MH}} = \gamma_t^{\mathrm{look}}(\cdot \mid q; \Pi_{\mathrm{II}})$, which incorporate future-value information. The MH step is applied only to low-reward duplicates produced during resampling—where lookahead is most beneficial—and uses the lookahead term $L_k$ in the acceptance ratio to directly assess the quality of future continuations. This ensures that proposed updates are accepted only when they lead to genuinely better trajectories, rather than simply swapping one weak prefix for another.

The full procedure is summarized in Algorithm~\ref{alg:general-block-resample-move-smc}. Some of the notation used include the proposal sampling $h$ blocks of size $B$ from the tempered base $\tilde p^{(\tau)}$ given a prefix $x_{1:kB}$:
$ Q_{\tau}^{(h)}\!\left(x_{kB+1:(k+h)B} \mid q, x_{1:kB}\right)
=\prod_{t=kB+1}^{(k+h)B}
\tilde p^{(\tau)}\!\left(x_t \mid q, x_{<t}\right).$ 
To estimate the lookahead term $\widehat L_k^{(H)}(x_{1:kB}, q; \Pi_{\mathrm{II}})$ we first sample $J$ continuations of prefix $x_{1:kB}$ each of length $H\cdot B$ from the proposal
$
\tilde x_{kB+1:(k+H)B}^{(j)} \sim Q_{\tau_{\mathrm{roll}}}^{(H)}(\cdot \mid q, x_{1:kB})
$ 
and then compute the ubiased lookahead estimate following \eqref{eq:lookahead-targetII-estimate} as

\begin{equation} \label{eq:alg-lookahead-estimate}
\widehat L_k(x_{1:kB}, q; \Pi_{\mathrm{II}})
=
\frac{1}{J}
\sum_{j=1}^{J}
\prod_{t=kB+1}^{(k+H)B}
\frac{
p\!\left(\tilde x_t^{(j)} \mid q, \tilde x_{<t}^{(j)}\right)^{\alpha}
\psi_t\!\left(\tilde x_{1:t}^{(j)}, q\right)
}{
\tilde p^{(\tau_{\mathrm{roll}})}\!\left(\tilde x_t^{(j)} \mid q, \tilde x_{<t}^{(j)}\right)
},
\quad
\tilde x_{1:kB}^{(j)} = x_{1:kB}.
\end{equation}
For each duplicated particle \(x^{(i)}_{1:kB}\) with low reward, the MH step proposes a new particle \(x'_{1:kB}\) by resampling its final blocks. The lookahead terms are estimated for both \(x^{(i)}_{1:kB}\) and \(x'_{1:kB}\), and the proposal is accepted with probability given by Lemma~\ref{lemma:targetII:mh} in Appendix~\ref{sec:blockwise-mh}.

We found this combination to be computationally efficient while achieving strong performance in practice. In general, our framework supports four variants for $\gamma^{\mathrm{SMC}}$ and $\gamma^{\mathrm{MH}}$: Target I or Target II, each paired with either prefix-only or lookahead intermediate targets.

\begin{algorithm}[t]
\caption{General reward-guided block-wise resample--move SMC with selective MH rejuvenation}
\label{alg:general-block-resample-move-smc}
\begin{algorithmic}[1]

\State \textbf{Input:} prompt $q$, max sequence length $T$,  block size $B$, number of blocks $K=T/B$, power parameter $\alpha$, number of particles $N$, ESS threshold $\tau_{\mathrm{ESS}}$, reward threshold $\tau_{\mathrm{R}}$, number of MH steps $S$, number of lookahead samples $J$, and lookahead horizon in terms of num.~of blocks $H$, lookahead samples rollout temperature $1/\tau_{\mathrm{roll}}$

\State Initialize $x_{1:0}^{(i)}=\emptyset$ and $W_0^{(i)}=1/N$ for $i=1,\dots,N$

\For{$k=1,\dots,K$}
    \State \fbox{\textsc{SMC Generation}}
    \For{$i=1,\dots,N$}
        %\State Sample block $x_{(k-1)B+1:kB}^{(i)} \sim \prod_{t=(k-1)B+1}^{kB} \tilde p^{\alpha}(x_t^{(i)} \mid q, x_{<t}^{(i)})$
        \State Sample block $x_{(k-1)B+1:kB}^{(i)} \sim Q_{\alpha}^{(1)}( q, x_{1:(k-1)B}^{(i)})$
        \State Extend particle $x_{1:kB}^{(i)} = \bigl(x_{1:(k-1)B}^{(i)},\,x_{(k-1)B+1:kB}^{(i)}\bigr)$
        \State Compute incremental block weight $w_k^{(i)} = \prod_{t=(k-1)B+1}^{kB} \psi_t(x_{1:t}^{(i)},q)$
        \State Update unnormalized weight $W_k^{(i)} = W_{k-1}^{(i)} w_k^{(i)}$
    \EndFor

    \State Normalize $\{W_k^{(i)}\}_{i=1}^N$ to obtain $\{\widetilde W_k^{(i)}\}_{i=1}^N$ and compute $\mathrm{ESS}_k = \bigl(\sum_{i=1}^N (\widetilde W_k^{(i)})^2\bigr)^{-1}$

    \If{$\mathrm{ESS}_k < \tau_{\mathrm{ESS}}$}
        \State \fbox{\textsc{SMC Resample}}
        \State Resample particles according to $\{\widetilde W_k^{(i)}\}_{i=1}^N$ and reset weights to $W_k^{(i)}=1/N$

        \State \fbox{\textsc{MH Rejuvenation}}
        \State Identify duplicated particles $\mathcal{D}_k$ after resampling
        \For{$i \in \mathcal D_k$ such that $R(x_{1:kB}^{(i)},q) < \tau_{\mathrm{R}}$}
            \For{$s=1,\dots,S$}
                \State Propose a new block $z'\sim Q_{\alpha}^{(1)}(\cdot\mid q, x^{(i)}_{1:(k-1)B})$
                \State Form proposed particle $x_{1:kB}' = \bigl(x_{1:(k-1)B}^{(i)},\,z'\bigr)$
                \State Estimate $\widehat{L}_k(x^{(i)}_{1:kB}, q;\Pi_{\mathrm{II}})$ and $\widehat{L}_k(x_{1:kB}',q,\Pi_{\mathrm{II}})$ using \eqref{eq:alg-lookahead-estimate}
                \State 
Accept $x_{1:kB}'$ with probability $a_k^{\mathrm{look}}\!(x_{1:kB}^{(i)},x'_{1:kB})$ equal to
\[
\min\left\{
1,\;
\frac{\widehat{L}_k\left(x'_{1:kB},q;\Pi_{\mathrm{II}}\right)}{\widehat{L}_k(x_{1:kB}^{(i)},q;\Pi_{\mathrm{II}})}
\prod_{t=(k-1)B+1}^{kB}
\frac{
p(x_t'\mid q,x'_{<t})^{\alpha}\,
}{
p(x_t^{(i)}\mid q,x_{<t}^{(i)})^{\alpha}\,
}
\cdot 
\frac{\psi_t\left(x'_{1:t},q\right)}{\psi_t(x_{1:t}^{(i)},q)}
\cdot
\frac{
\tilde p^{(\alpha)}(x_t^{(i)}\mid q,x_{<t}^{(i)})
}{
\tilde p^{(\alpha)}\left(x_t'\mid q,x'_{<t}\right)
}
% \prod_{t=(k-1)B+1}^{kB}
% \frac{
% \psi_t(x'_{1:t},q)\, Z_t(x_{<t}')
% }{
% \psi_t(x_{1:t},q)\, Z_t(x_{<t})
% }
\right\}
\]
\State If accepted, set $x_{1:kB}^{(i)} \gets x_{1:kB}'$
            \EndFor
        \EndFor
    \EndIf
\EndFor

\State \textbf{Return:} weighted particle approximation $\{(x_{1:T}^{(i)}, W_K^{(i)})\}_{i=1}^N$

\end{algorithmic}
\end{algorithm}

\section{Experiments \& Results}
\subsection{Setup}
\paragraph{Implementation Details.}
We implement all methods in a unified inference-time sampling framework built on \texttt{vLLM}.  Following the~\cite{ji2026scalable} all benchmark experiment settings, we cap the maximum generation length at $T=3072$ tokens and terminate early upon emitting an \texttt{EOS} token. Our method (\textbf{SMC reward-guided lookahead}) combines reward-guided SMC (prefix potentials only) with a resample-move Metropolis--Hastings rejuvenation step whose acceptance ratio uses rollout-based lookahead with estimated lookahead $L_t$. Concretely, SMC iteratively extends a population of partial solutions, reweights particles using a task-specific verifier signal, and periodically resamples and rejuvenates particles to maintain diversity while concentrating probability mass on high-quality trajectories. The details are in Algorithm \ref{alg:general-block-resample-move-smc}.
\setlength{\parskip}{0pt}
\paragraph{Task-specific hyperparameters and reward signals.} 
We apply a unified method across all tasks, keeping most hyperparameters constant while adapting the reward verifier, block size ($B$), and rollout temperature ($\tau_{\mathrm{roll}}$) to best capture the contextual needs of each domain. Globally, we default to $\alpha=4.0$, $N=16$, $J=2$ lookahead rollouts, and $S=2$ Metropolis-Hastings rejuvenation steps per resampling stage. For \textbf{MATH500} and \textbf{GPQA}, we set a larger $B=512$ and $\tau_{\mathrm{roll}}=0.3$. This coarser granularity reflects the nature of mathematical and scientific reasoning, which often requires longer, continuous chains of thought before Process Reward Models (\texttt{Act-X~\cite{duan2025efficient}} and \texttt{ThinkPRM-1.5B}, respectively) can assign meaningful reward signals. In contrast, for \textbf{HumanEval}, we use a much finer-grained $B=64$ and $\tau_{\mathrm{roll}}=0.1$. Because code generation can be evaluated in smaller logical chunks, this frequent resampling fully leverages our method's capacity for early course correction, utilizing unit-test execution (5s timeout) alongside an auxiliary syntax reward (weight 0.3).
\setlength{\parskip}{0pt}
\paragraph{Base Models and Benchmarks.}
We evaluate our framework on three base language models: \texttt{Qwen2.5-7B}, \texttt{Qwen2.5-Math-7B}, and \texttt{DeepSeek-Math-7B}. To comprehensively assess performance across diverse reasoning domains, we consider three benchmarks: \textbf{MATH500} (measuring final-answer exact match), \textbf{HumanEval} (evaluating functional correctness via standard pass@1 execution against official Python unit tests), and the \textbf{GPQA} diamond split (reporting multiple-choice accuracy).
\setlength{\parskip}{0pt}
\paragraph{Baselines.}
We keep the baseline suite aligned with prior work on power sampling following~\cite{ji2026scalable}. Concretely, we compare against \textbf{Base} (standard decoding with temperature set to one), \textbf{Low-temperature} sampling (with temperature $1/\alpha$), \textbf{Best-of-N}, \textbf{MCMC Power Sampling}~\citep{karan2025reasoning}, \textbf{Scalable Power Sampling}~\citep{ji2026scalable}, and \textbf{Power-SMC}~\citep{azizi2026power}. In addition, we include an SMC-style reward-guided baseline \textbf{SMC (reward)}, implemented following \citep{lew2023sequential}, which uses prefix-only intermediate targets with the same reward verifiers and hyperparameters $(N, B, \alpha)$ as our full method, but omits lookahead and MH rejuvenation. This baseline serves as a direct ablation isolating the contribution of our lookahead and resample--move components. We also report \textbf{GRPO} when available.
\setlength{\parskip}{0pt}
\paragraph{Results.}
Table~\ref{tab:main_results} reports the pass@1 accuracies on MATH500, HumanEval, and GPQA. Our proposed SMC reward-guided lookahead consistently sets a new state-of-the-art for inference-time decoding across all three 7B models. On HumanEval, our method reaches peak pass@1 scores of 0.878 (Qwen2.5-7B), 0.854 (Qwen2.5-Math-7B), and 0.781 (DeepSeek-Math-7B), yielding massive absolute improvements of up to +54.9\% over base models and substantially outperforming recent strong baselines like Scalable Power Sampling~\cite{ji2026scalable}. In mathematical reasoning (MATH500), our approach attains 0.790 and 0.604 on Qwen2.5-7B and DeepSeek-Math-7B, respectively, effectively surpassing even domain-specific RLHF post-training methods such as GRPO. We observe similarly robust gains on the GPQA scientific reasoning benchmarks, achieving up to 0.424 pass@1.
\setlength{\parskip}{0pt}
\paragraph{Effect of the lookahead mechanism (Ablation).} Crucially, comparing our full method to the prefix-only SMC (reward)~\cite{lew2023sequential} explicitly isolates the impact of our lookahead formulation. Across the board, incorporating intermediate lookahead targets yields striking performance leaps. On code generation (HumanEval), the lookahead mechanism drives absolute gains of +9.1\%, +10.4\%, and +15.3\% over standard SMC with same reward model for Qwen2.5-7B, Qwen2.5-Math-7B, and DeepSeek-Math-7B, respectively. We observe a parallel trend on MATH500, with lookahead providing up to an +8.8\% absolute boost (DeepSeek-Math-7B) over the prefix-only variant. This ablation highlights a fundamental insight: while prefix potentials improve upon base models, they frequently lead the sampler into local optima during long-horizon generation. By estimating the future reward landscape through intermediate marginals, the lookahead mechanism successfully prevents myopic sampling, steering the model toward globally optimal sequence trajectories.

\begin{wrapfigure}{r}{0.5\columnwidth}
\vspace{-1.2em}
\centering
\includegraphics[width=0.48\columnwidth]{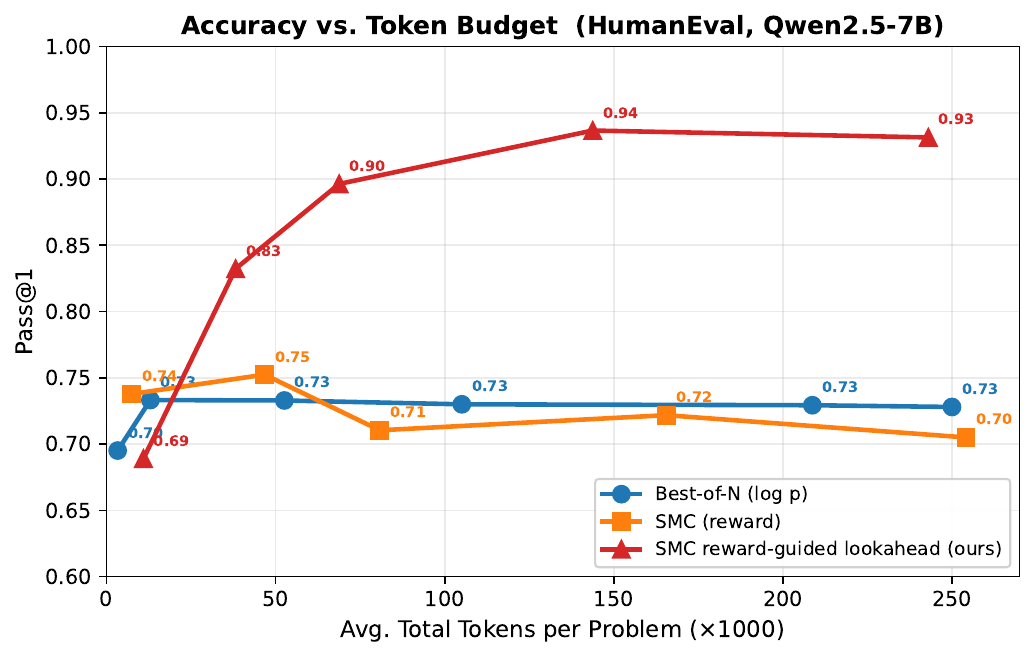}
\caption{Pass@1 vs.\ total tokens per problem on a subset of 82 HumanEval tasks (Qwen2.5-7B). Best-of-N and SMC~(reward) saturate early. Our method scales monotonically with compute.}
\label{fig:token_tradeoff}
\vspace{-1em}
\end{wrapfigure}
\newpage
\paragraph{Token Budget vs.\ Accuracy.}
A natural question is whether our gains arise simply from generating more tokens. Figure~\ref{fig:token_tradeoff} addresses this by sweeping the number of particles (or samples for Best-of-$N$) from $2$ to $128$ on HumanEval. Best-of-$N$ (ranked by $\log p(x\!\mid\!q)$) and SMC~(reward) both plateau early at $0.73$ pass@1, with no improvement beyond $N{=}16$, consistent with diminishing returns from independent sampling. In contrast, our method continues to improve, reaching $0.94$ pass@1 at ${\sim}144\text{K}$ tokens per problem. This demonstrates that the gains are not explained by higher token consumption, but by how compute is allocated: the lookahead $\zeta$ term is invoked selectively during resampling, and additional budget is directed toward lookahead estimation and MH rejuvenation of low-reward particles, yielding a qualitatively different scaling behavior.

\begin{table}[t]
\centering
\small
\setlength{\tabcolsep}{10pt}
\renewcommand{\arraystretch}{1.15}
\resizebox{0.95\columnwidth}{!}{%
\begin{tabular}{lccc}
\toprule
 & \textbf{MATH500} & \textbf{HumanEval} & \textbf{GPQA} \\
\midrule
\multicolumn{4}{l}{\textbf{Qwen2.5-7B}}\\
\hspace{1.5em}Base & 0.498 & 0.329 & 0.278 \\
\hspace{1.5em}Low-temperature & 0.628 & 0.524 & 0.303 \\
\hspace{1.5em}Best-of-N & 0.650 & 0.609 & 0.282 \\
\hspace{1.5em}SMC (reward)~\citep{lew2023sequential} & 0.710 & 0.787 & 0.323 \\
\hspace{1.5em}MCMC Power Sampling~\citep{karan2025reasoning} & 0.706 & 0.622 & 0.318 \\
\hspace{1.5em}Scalable Power Sampling~\citep{ji2026scalable} & 0.708 & 0.756 & 0.349 \\
\hspace{1.5em}Power-SMC~\citep{azizi2026power} & 0.716 & 0.683 & 0.313 \\
% \hspace{1.5em}\textbf{SMC reward-guided lookahead (ours)} & \textbf{0.784} & \textbf{0.878} & \textbf{0.384} \\
\hspace{1.5em}\textbf{SMC reward-guided lookahead (ours)} & \textbf{0.790} & \textbf{0.878} & \textbf{0.384} \\
\midrule
\hspace{1.5em}GRPO (MATH) & 0.740 & 0.561 & 0.354 \\
\midrule
\multicolumn{4}{l}{\textbf{Qwen2.5-Math-7B}}\\
\hspace{1.5em}Base & 0.496 & 0.329 & 0.278 \\
\hspace{1.5em}Low-temperature & 0.690 & 0.512 & 0.353 \\
\hspace{1.5em}Best-of-N & 0.684 & 0.512 & 0.343 \\
\hspace{1.5em}SMC (reward)~\citep{lew2023sequential} & 0.756 & 0.750 & 0.384 \\
\hspace{1.5em}MCMC Power Sampling~\citep{karan2025reasoning} & 0.748 & 0.573 & 0.389 \\
\hspace{1.5em}Scalable Power Sampling~\citep{ji2026scalable} & 0.758 & 0.604 & 0.409 \\
\hspace{1.5em}Power-SMC~\citep{azizi2026power} & 0.742 & 0.585 & 0.349 \\
% \hspace{1.5em}\textbf{SMC reward-guided lookahead (ours)} & \textbf{0.778} & \textbf{0.854} & \textbf{0.424} \\
\hspace{1.5em}\textbf{SMC reward-guided lookahead (ours)} & \textbf{0.782} & \textbf{0.854} & \textbf{0.424} \\
\midrule
\hspace{1.5em}GRPO (MATH) & 0.785 & 0.537 & 0.399 \\
\midrule
\multicolumn{4}{l}{\textbf{DeepSeek-Math-7B}}\\
\hspace{1.5em}Base & 0.362 & 0.415 & 0.333 \\
\hspace{1.5em}Low-temperature & 0.366 & 0.427 & \textbf{0.430} \\
\hspace{1.5em}Best-of-N & 0.420 & 0.433 & 0.338 \\
\hspace{1.5em}SMC (reward)~\citep{lew2023sequential} & 0.516 & 0.628 & 0.388 \\
\hspace{1.5em}MCMC Power Sampling~\citep{karan2025reasoning} & 0.424 & 0.470 & 0.345 \\
\hspace{1.5em}Scalable Power Sampling~\citep{ji2026scalable} & 0.464 & 0.487 & 0.364 \\
\hspace{1.5em}Power-SMC~\citep{azizi2026power} & 0.456 & 0.475 & 0.326 \\
% \hspace{1.5em}\textbf{SMC reward-guided lookahead (ours)} & \textbf{0.604} & \textbf{0.781} & 0.424 \\
\hspace{1.5em}\textbf{SMC reward-guided lookahead (ours)} & \textbf{0.604} & \textbf{0.781} & 0.424 \\
\midrule
\hspace{1.5em}GRPO (MATH) & 0.492 & 0.524 & 0.333 \\
\bottomrule
\end{tabular}
}
\caption{Main results (pass@1) on MATH500, HumanEval, and GPQA (diamond split). Baseline numbers follow prior work~\cite{ji2026scalable}; we add our method as an extra row.}
\label{tab:main_results}
\end{table}

\section{Conclusion and Future Work}

We introduced a principled probabilistic framework for reward-guided decoding in large language models by defining reward-augmented sequence distributions within a Feynman--Kac formulation. Building on this perspective, we developed Sequential Monte Carlo methods with both prefix-only and lookahead intermediate targets, showing that lookahead yields optimal weighting by matching exact marginals. Our resample--move framework further combines efficient exploration with targeted Metropolis--Hastings corrections, enabling practical and scalable inference. Empirically, the proposed approach consistently outperforms existing decoding and sampling baselines across diverse reasoning tasks, highlighting the importance of sequence-level modeling and principled sampling.

Looking ahead, several directions merit further exploration. A key priority is improving the efficiency of lookahead estimation. One promising approach is adaptive compute allocation, where rollout budget is distributed based on particle uncertainty, with early stopping when decisions are clear and increased compute devoted to high-impact, ambiguous particles. Finally, our method is naturally suited for generating high-quality samples, making it a compelling tool for data generation in reinforcement learning for reasoning tasks, as well as more complex agentic applications.

\bibliography{colm2026_conference}

@article{ji2026scalable,
  title={Scalable Power Sampling: Unlocking Efficient, Training-Free Reasoning for LLMs via Distribution Sharpening},
  author={Ji, Xiaotong and Tutunov, Rasul and Zimmer, Matthieu and Ammar, Haitham Bou},
  journal={arXiv preprint arXiv:2601.21590},
  year={2026}
}

@article{azizi2026power,
  title={Power-SMC: Low-Latency Sequence-Level Power Sampling for Training-Free LLM Reasoning},
  author={Azizi, Seyedarmin and Potraghloo, Erfan Baghaei and Ahmadi, Minoo and Kundu, Souvik and Pedram, Massoud},
  journal={arXiv preprint arXiv:2602.10273},
  year={2026}
}

@article{karan2025reasoning,
  title={Reasoning with sampling: Your base model is smarter than you think},
  author={Karan, Aayush and Du, Yilun},
  journal={arXiv preprint arXiv:2510.14901},
  year={2025}
}

@article{lightman2023step,
  title={Let's Verify Step by Step},
  author={Lightman, Hunter and Kosaraju, Vineet and Burda, Yura and Edwards, Harri and Baker, Bowen and Lee, Teddy and Leike, Jan and Schulman, John and Sutskever, Ilya and Cobbe, Karl},
  journal={arXiv preprint arXiv:2305.20050},
  year={2023},
}

@article{ouyang2022training,
  title={Training language models to follow instructions with human feedback},
  author={Ouyang, Long and Wu, Jeffrey and Jiang, Xu and Almeida, Diogo and Wainwright, Carroll and Mishkin, Pamela and Zhang, Chong and Agarwal, Sandhini and Slama, Katarina and Ray, Alex and others},
  journal={Advances in neural information processing systems},
  volume={35},
  pages={27730--27744},
  year={2022}
}

@inproceedings{holtzman2020nucleus,
  title={The Curious Case of Neural Text Degeneration},
  author={Holtzman, Ari and Buys, Jan and Du, Li and Forbes, Maxwell and Choi, Yejin},
  booktitle={International Conference on Learning Representations (ICLR)},
  year={2020}
}

@article{vijayakumar2016diverse,
  title={Diverse beam search: Decoding diverse solutions from neural sequence models},
  author={Vijayakumar, Ashwin K and Cogswell, Michael and Selvaraju, Ramprasath R and Sun, Qing and Lee, Stefan and Crandall, David and Batra, Dhruv},
  journal={arXiv preprint arXiv:1610.02424},
  year={2016}
}

@article{cobbe2021training,
  title={Training verifiers to solve math word problems},
  author={Cobbe, Karl and Kosaraju, Vineet and Bavarian, Mohammad and Chen, Mark and Jun, Heewoo and Kaiser, Lukasz and Plappert, Matthias and Tworek, Jerry and Hilton, Jacob and Nakano, Reiichiro and others},
  journal={arXiv preprint arXiv:2110.14168},
  year={2021}
}

@article{lipkin2025fast,
  title={Fast controlled generation from language models with adaptive weighted rejection sampling},
  author={Lipkin, Benjamin and LeBrun, Benjamin and Vigly, Jacob Hoover and Loula, Jo{\~a}o and MacIver, David R and Du, Li and Eisner, Jason and Cotterell, Ryan and Mansinghka, Vikash and O'Donnell, Timothy J and others},
  journal={arXiv preprint arXiv:2504.05410},
  year={2025}
}

@article{lew2023sequential,
  title={Sequential monte carlo steering of large language models using probabilistic programs},
  author={Lew, Alexander K and Zhi-Xuan, Tan and Grand, Gabriel and Mansinghka, Vikash K},
  journal={arXiv preprint arXiv:2306.03081},
  year={2023}
}

@article{loula2025syntactic,
  title={Syntactic and semantic control of large language models via sequential monte carlo},
  author={Loula, Jo{\~a}o and LeBrun, Benjamin and Du, Li and Lipkin, Ben and Pasti, Clemente and Grand, Gabriel and Liu, Tianyu and Emara, Yahya and Freedman, Marjorie and Eisner, Jason and others},
  journal={arXiv preprint arXiv:2504.13139},
  year={2025}
}

@article{uesato2022solving,
  title={Solving math word problems with process-and outcome-based feedback},
  author={Uesato, Jonathan and Kushman, Nate and Kumar, Ramana and Song, Francis and Siegel, Noah and Wang, Lisa and Creswell, Antonia and Irving, Geoffrey and Higgins, Irina},
  journal={arXiv preprint arXiv:2211.14275},
  year={2022}
}

@article{bai2022constitutional,
  title={Constitutional ai: Harmlessness from ai feedback},
  author={Bai, Yuntao and Kadavath, Saurav and Kundu, Sandipan and Askell, Amanda and Kernion, Jackson and Jones, Andy and Chen, Anna and Goldie, Anna and Mirhoseini, Azalia and McKinnon, Cameron and others},
  journal={arXiv preprint arXiv:2212.08073},
  year={2022}
}

@article{chow2024bon,
  title={Inference-aware fine-tuning for best-of-n sampling in large language models},
  author={Chow, Yinlam and Tennenholtz, Guy and Gur, Izzeddin and Zhuang, Vincent and Dai, Bo and Thiagarajan, Sridhar and Boutilier, Craig and Agarwal, Rishabh and Kumar, Aviral and Faust, Aleksandra},
  journal={arXiv preprint arXiv:2412.15287},
  year={2024}
}

@article{setlur2024processverifiers,
  title={Rewarding progress: Scaling automated process verifiers for llm reasoning},
  author={Setlur, Amrith and Nagpal, Chirag and Fisch, Adam and Geng, Xinyang and Eisenstein, Jacob and Agarwal, Rishabh and Agarwal, Alekh and Berant, Jonathan and Kumar, Aviral},
  journal={arXiv preprint arXiv:2410.08146},
  year={2024}
}

@article{ibrahim2000power,
  title={Power prior distributions for regression models},
  author={Ibrahim, Joseph G and Chen, Ming-Hui},
  journal={Statistical Science},
  pages={46--60},
  year={2000},
  publisher={JSTOR}
}

@misc{radford2018improving,
  title={Improving language understanding by generative pre-training},
  author={Radford, Alec and Narasimhan, Karthik and Salimans, Tim and Sutskever, Ilya and others},
  year={2018},
  publisher={San Francisco, CA, USA},
  note={OpenAI technical report}
}

@article{vaswani2017attention,
  title={Attention is all you need},
  author={Vaswani, Ashish and Shazeer, Noam and Parmar, Niki and Uszkoreit, Jakob and Jones, Llion and Gomez, Aidan N and Kaiser, {\L}ukasz and Polosukhin, Illia},
  journal={Advances in neural information processing systems},
  volume={30},
  year={2017}
}

@inproceedings{zhang2025lessons,
  title={The lessons of developing process reward models in mathematical reasoning},
  author={Zhang, Zhenru and Zheng, Chujie and Wu, Yangzhen and Zhang, Beichen and Lin, Runji and Yu, Bowen and Liu, Dayiheng and Zhou, Jingren and Lin, Junyang},
  booktitle={Findings of the Association for Computational Linguistics: ACL 2025},
  pages={10495--10516},
  year={2025}
}

@article{gordon1993novel,
  title={Novel approach to nonlinear/non-Gaussian Bayesian state estimation},
  author={Gordon, Neil J and Salmond, David J and Smith, Adrian FM},
  journal={IEE Proceedings F (Radar and Signal Processing)},
  volume={140},
  number={2},
  pages={107--113},
  year={1993}
}

@book{doucet2001sequential,
  title={Sequential Monte Carlo Methods in Practice},
  author={Doucet, Arnaud and de Freitas, Nando and Gordon, Neil},
  publisher={Springer},
  year={2001}
}

@article{del2006sequential,
  title={Sequential monte carlo samplers},
  author={Del Moral, Pierre and Doucet, Arnaud and Jasra, Ajay},
  journal={Journal of the Royal Statistical Society Series B: Statistical Methodology},
  volume={68},
  number={3},
  pages={411--436},
  year={2006},
  publisher={Oxford University Press}
}

@article{chopin2004central,
  title={Central limit theorem for sequential Monte Carlo methods and its application to Bayesian inference},
  author={Chopin, Nicolas},
  journal={The Annals of Statistics},
  volume={32},
  number={6},
  pages={2385--2411},
  year={2004}
}

@incollection{doucet2011tutorial,
  title={A Tutorial on Particle Filtering and Smoothing: Fifteen Years Later},
  author={Doucet, Arnaud and Johansen, Adam M.},
  booktitle={The Oxford Handbook of Nonlinear Filtering},
  editor={Crisan, Dan and Rozovskii, Boris},
  pages={656--704},
  year={2011},
  publisher={Oxford University Press}
}

@article{wang2020contextual,
  title={Contextual temperature for language modeling},
  author={Wang, Pei-Hsin and Hsieh, Sheng-Iou and Chang, Shih-Chieh and Chen, Yu-Ting and Pan, Jia-Yu and Wei, Wei and Juan, Da-Chang},
  journal={arXiv preprint arXiv:2012.13575},
  year={2020}
}

@article{brown2020language,
  title={Language models are few-shot learners},
  author={Brown, Tom and Mann, Benjamin and Ryder, Nick and Subbiah, Melanie and Kaplan, Jared D and Dhariwal, Prafulla and Neelakantan, Arvind and Shyam, Pranav and Sastry, Girish and Askell, Amanda and others},
  journal={Advances in neural information processing systems},
  volume={33},
  pages={1877--1901},
  year={2020}
}

@article{huang2025best,
  title={Is best-of-n the best of them? coverage, scaling, and optimality in inference-time alignment},
  author={Huang, Audrey and Block, Adam and Liu, Qinghua and Jiang, Nan and Krishnamurthy, Akshay and Foster, Dylan J},
  journal={arXiv preprint arXiv:2503.21878},
  year={2025}
}

@article{neal2001annealed,
  title={Annealed importance sampling},
  author={Neal, Radford M},
  journal={Statistics and computing},
  volume={11},
  number={2},
  pages={125--139},
  year={2001},
  publisher={Springer}
}

@article{gilks2001following,
  title={Following a moving target—Monte Carlo inference for dynamic Bayesian models},
  author={Gilks, Walter R and Berzuini, Carlo},
  journal={Journal of the Royal Statistical Society: Series B},
  volume={63},
  number={1},
  pages={127--146},
  year={2001}
}

@article{snell2024scaling,
  title={Scaling llm test-time compute optimally can be more effective than scaling model parameters},
  author={Snell, Charlie and Lee, Jaehoon and Xu, Kelvin and Kumar, Aviral},
  journal={arXiv preprint arXiv:2408.03314},
  year={2024}
}

@article{rafailov2023direct,
  title={Direct preference optimization: Your language model is secretly a reward model},
  author={Rafailov, Rafael and Sharma, Archit and Mitchell, Eric and Manning, Christopher D and Ermon, Stefano and Finn, Chelsea},
  journal={Advances in neural information processing systems},
  volume={36},
  pages={53728--53741},
  year={2023}
}

@book{moral2004feynman,
  title={Feynman-Kac formulae: genealogical and interacting particle systems with applications},
  author={Moral, Pierre},
  year={2004},
  publisher={Springer}
}

@article{whiteley2012twisted,
  author    = {Nick Whiteley and Anthony Lee},
  title     = {Twisted Particle Filters},
  journal   = {Annals of Statistics},
  volume    = {42},
  number    = {1},
  pages     = {115--141},
  year      = {2014},
  doi       = {10.1214/13-AOS1167}
}

@article{zhao2024probabilistic,
  title={Probabilistic inference in language models via twisted sequential monte carlo},
  author={Zhao, Stephen and Brekelmans, Rob and Makhzani, Alireza and Grosse, Roger},
  journal={arXiv preprint arXiv:2404.17546},
  year={2024}
}

@article{park2024grammar,
  title={Grammar-aligned decoding},
  author={Park, Kanghee and Wang, Jiayu and Berg-Kirkpatrick, Taylor and Polikarpova, Nadia and D'Antoni, Loris},
  journal={Advances in Neural Information Processing Systems},
  volume={37},
  pages={24547--24568},
  year={2024}
}

@article{duan2025efficient,
  title={Efficient process reward model training via active learning},
  author={Duan, Keyu and Liu, Zichen and Mao, Xin and Pang, Tianyu and Chen, Changyu and Chen, Qiguang and Shieh, Michael Qizhe and Dou, Longxu},
  journal={arXiv preprint arXiv:2504.10559},
  year={2025}
}
\bibliographystyle{colm2026_conference}

\appendix

\section{Additional Methodological Details}

\subsection{Conditional Distributions} \label{sec:conditional:distributions}

The following lemma gives the exact conditional next-token distribution induced by the unified target $\Pi$ in \eqref{eq:unified-full-target}. The exact conditional distributions for Target I and Target II then follow by substituting the corresponding choice of transition factor $m_t$.

\begin{lemma}[Conditional next-token distribution under the unified target]
Consider the unified target $\Pi$ from \eqref{eq:unified-full-target}. Then the conditional distribution of the next token given a prefix $x_{<t}$ is
\begin{equation}
\Pi(x_t\mid q,x_{<t})
=
\frac{
m_t(x_t\mid q,x_{<t})
\psi_t(x_{1:t},q)
L_t(x_{1:t},q;\Pi)
}{
\sum_{v\in \mathcal{V}}
m_t(v\mid q,x_{<t})
\psi_t(x_{<t},v,q)
L_t(x_{<t},v,q;\Pi)
},
\end{equation}
where the lookahead term $L_t$ is defined in \eqref{eq:generic-lookahead-term-lemma} and $\mathcal{V}$ is the whole token vocabulary.
\end{lemma}

\begin{proof}
By definition,
\begin{equation}
\Pi(x_t\mid q,x_{<t})
=
\frac{
\sum_{x_{t+1:T}} \Pi(x_{1:T}\mid q)
}{
\sum_{v\in \mathcal{V}}\sum_{x_{t+1:T}} \Pi(x_{<t},v,x_{t+1:T}\mid q)
}.
\end{equation}
Substituting the expression for $\Pi$ gives
\begin{align}
\Pi(x_t\mid q,x_{<t})
&=
\frac{
\sum_{x_{t+1:T}}
\prod_{s=1}^{T} m_s(x_s\mid q,x_{<s})
\prod_{s=1}^{T} \psi_s(x_{1:s},q)
}{
\sum_{v\in \mathcal{V}}\sum_{x_{t+1:T}}
\prod_{s=1}^{T} m_s(x_s\mid q,x_{<s})
\prod_{s=1}^{T} \psi_s(x_{1:s},q)
}.
\end{align}
All factors depending only on the prefix $x_{<t}$ cancel in the ratio. Hence
\begin{align}
\Pi(x_t\mid q,x_{<t})
&=
\frac{
m_t(x_t\mid q,x_{<t})
\psi_t(x_{1:t},q)
\sum_{x_{t+1:T}}
\prod_{s=t+1}^{T} \left(m_s(x_s\mid q,x_{<s})\psi_s(x_{1:s},q)\right)
}{
\sum_{v\in \mathcal{V}}
m_t(v\mid q,x_{<t})
\psi_t(x_{<t},v,q)
\sum_{x_{t+1:T}}
\prod_{s=t+1}^{T} \left(m_s(x_s\mid q,x_{<s}) \psi_s(x_{1:s},q)\right)
}.
\end{align}
Plugging in the definition of $L_t$ above yields the stated formula.
\end{proof}

\paragraph{Conditional Distribution for Target I.}
For Target I, $m_t(x_t\mid q,x_{<t}) = \tilde p_t^{(\alpha)}(x_t\mid q,x_{<t}).$ Therefore,
\begin{equation}
\Pi_{\mathrm{I}}(x_t\mid q,x_{<t})
=
\frac{
\tilde p_t^{(\alpha)}(x_t\mid q,x_{<t})
\psi_t(x_{1:t},q)
L_t(x_{1:t},q;\Pi_{\mathrm{I}})
}{
\sum_{v\in \mathcal{V}}
\tilde p_t^{(\alpha)}(v\mid q,x_{<t})
\psi_t(x_{<t},v,q)
L_t(x_{<t},v,q;\Pi_{\mathrm{I}})
},
\end{equation}
where $L_t$ term is defined in \eqref{eq:targetI:Lt}.

\paragraph{Conditional Distribution for Target II.}
For Target II, $m_t(x_t\mid q,x_{<t})=p(x_t\mid q,x_{<t})^{\alpha}.$
Therefore,
\begin{equation}
\Pi_{\mathrm{II}}(x_t\mid q,x_{<t})
=
\frac{
p(x_t\mid q,x_{<t})^{\alpha}
\psi_t(x_{1:t},q)
L_t(x_{1:t},q;\Pi_{\mathrm{II}})
}{
\sum_{v\in \mathcal{V}}
p(v\mid q,x_{<t})^{\alpha}
\psi_t(x_{<t},v,q)
L_t(x_{<t},v,q;\Pi_{\mathrm{II}})
},
\end{equation}
where $L_t$ term is defined in \eqref{eq:targetII:Lt}.

%%%%%%%%%%%%%%%%%%%%%%%%%%%%%%%%%%%%%%%%%%%%%%%%%%%%%%%%
%%%%%%%%%%%%%%%%%%%%%%%%%%%%%%%%%%%%%%%%%%%%%%%%%%%%%%%%
%%%%%%%%%%%%%%%%%%%%%%%%%%%%%%%%%%%%%%%%%%%%%%%%%%%%%%%%
%%%%%%%%%%%%%%%%%%%%%%%%%%%%%%%%%%%%%%%%%%%%%%%%%%%%%%%%

\subsection{SMC for Targets I and II} \label{sec:instantiation}

\paragraph{Prefix-only SMC weights}
Let $r_t(x_t\mid q,x_{<t})$ be the proposal distribution used to sample the next token. Then the incremental importance weight is
\begin{equation}
\label{eq:generic-simple-weight}
w_t^{\mathrm{prf}}(x_{1:t}, q; \Pi)
=
\frac{m_t(x_t\mid q,x_{<t})}
{r_t(x_t\mid q,x_{<t})}
\psi_t(x_{1:t},q).
\end{equation}

\paragraph{Prefix-only SMC instantiation for Target I.}

If $r_t(x_t\mid q,x_{<t})=\tilde p_t^{(\beta)}(x_t\mid q,x_{<t})$,
then
\begin{equation}
w_t^{\mathrm{prf}}(x_{1:t}; \Pi_{\mathrm{I}})
=\frac{\tilde p_t^{(\alpha)}(x_t\mid q,x_{<t})}
{\tilde p_t^{(\beta)}(x_t\mid q,x_{<t})}
\psi_t(x_{1:t},q).
\end{equation}
In particular, when $\beta=\alpha$,
$w_t^{\mathrm{prf}}(x_{1:t}; \Pi_{\mathrm{I}})=\psi_t(x_{1:t},q)$ so the weight updates depend only on the reward potentials and not the base model.

\paragraph{Prefix-only SMC instantiation for Target II.} 
If the proposal is $\tilde p_t^{(\beta)}(\cdot\mid q,x_{<t})$, then
\begin{equation}
w_t^{\mathrm{prf}}(x_{1:t}; \Pi_{\mathrm{II}})
=
\frac{p(x_t\mid q,x_{<t})^{\alpha}}
{\tilde p_t^{(\beta)}(x_t\mid q,x_{<t})}
\psi_t(x_{1:t},q).
\end{equation}
In particular, for $\alpha=\beta$, this becomes $w_t^{\mathrm{prf}}(x_{1:t}; \Pi_{\mathrm{II}})=Z_t(x_{<t})\psi_t(x_{1:t},q)$, where $Z_t(x_{<t}) = \sum_{v\in\mathcal{V}} p(v\mid q, x_{<t})^{\alpha}$, so the SMC weights in this case depend on the base model as well as the reward potentials.

\paragraph{Lookahead SMC weights.}
Let $r_t(x_t\mid q,x_{<t})$ be a proposal distribution. Then the incremental importance weight is
\begin{equation}
\label{eq:generic-lookahead-weight-lemma}
w_t^{\mathrm{look}}(x_{1:t}, q; \Pi)
=
\frac{m_t(x_t\mid q,x_{<t})}
{r_t(x_t\mid q,x_{<t})}
\psi_t(x_{1:t},q)
\frac{L_t(x_{1:t},q; \Pi)}{L_{t-1}(x_{1:t-1},q; \Pi)}.
\end{equation}

\paragraph{Lookahead SMC instantiation for Target I}

If the proposal is $r_t=\tilde p_t^{(\beta)}(\cdot\mid q,x_{<t})$, then
\begin{equation}
w_t^{\mathrm{look}}(x_{1:t}; \Pi_{\mathrm{I}})
=
\frac{\tilde p_t^{(\alpha)}(x_t\mid q,x_{<t})}
{\tilde p_t^{(\beta)}(x_t\mid q,x_{<t})}
\psi_t(x_{1:t},q)
\frac{L_t(x_{1:t},q; \Pi_{\mathrm{I}})}{L_{t-1}(x_{1:t-1},q; \Pi_{\mathrm{I}})}.
\end{equation}
When $\beta=\alpha$, the first fractional term above equals to one and $w_t$ has a simpler expression depending on the reward potential and the ratio of lookahead terms.

\paragraph{Lookahead SMC instantiation for Target II}

If the proposal is $\tilde p_t^{(\beta)}(\cdot\mid q,x_{<t})$, then
\begin{equation}
w_t^{\mathrm{look}}(x_{1:t}; \Pi_{\mathrm{II}})
=
\frac{p(x_t\mid q,x_{<t})^{\alpha}}
{\tilde p_t^{(\beta)}(x_t\mid q,x_{<t})}
\psi_t(x_{1:t},q)
\frac{L_t(x_{1:t},q; \Pi_{\mathrm{II}})}{L_{t-1}(x_{1:t-1},q; \Pi_{\mathrm{II}})}.
\end{equation}

\subsection{Mean-square error of prefix and lookahead weights} \label{section:mse}

The theorem below establishes that the exact lookahead weight is the conditional expectation of the full importance weight and therefore minimizes mean-square error among all \(\mathcal F_t\)-measurable approximations. The prefix weight is suboptimal by an explicit excess error term that measures the effect of neglecting future reward contributions. The approximate lookahead weight recovers the exact lookahead behavior up to a Monte Carlo error term that vanishes as the number of lookahead samples increases. Overall, the result formalizes the advantage of lookahead-based weighting over prefix-only weighting.

\begin{theorem}[Mean-square error of prefix, exact lookahead, and approximate lookahead weights] \label{theorem:mse}
Let
\[
W_T=\prod_{s=1}^T
\frac{m_s(X_s\mid q,X_{<s})}{r_s(X_s\mid q,X_{<s})}\,
\psi_s(X_{1:s},q)
\] 
be the full importance weight of a trajectory sampled from $R(x_{1:T}\mid q)=\prod_{s=1}^T r_s(x_s\mid q,x_{<s}).$
For \(t\le T\), define the prefix weight
\[
W_t^{\mathrm{prf}}
:=\prod_{s=1}^t
\frac{m_s(X_s\mid q,X_{<s})}{r_s(X_s\mid q,X_{<s})}\,
\psi_s(X_{1:s},q),
\]
and the exact lookahead weight
$W_t^{\mathrm{look}}
:=
W_t^{\mathrm{prf}}\,L_t(X_{1:t},q;\Pi),$ where $L_t$ is defined in \eqref{eq:generic-lookahead-term-lemma}.
Let \(\mathcal F_t=\sigma(X_{1:t})\). Then
\[
W_t^{\mathrm{look}}=\mathbb E[W_T\mid \mathcal F_t].
\]
Further, let \(\widehat L_t\) be a Monte Carlo estimate of \(L_t\) satisfying \(\mathbb E[\widehat L_t\mid \mathcal F_t]=L_t\) and $\mathbb E\!\left[(\widehat L_t-L_t)^2\mid \mathcal F_t\right]\le C_t/J$. Define the approximate lookahead weight
$W_t^{\mathrm{app}}:=W_t^{\mathrm{prf}}\,\widehat L_t.$ Then:
\begin{align}
&\mathbb E\!\left[(W_T-W_t^{\mathrm{look}})^2\right]
=\mathbb E\!\left[\mathrm{Var}(W_T\mid \mathcal F_t)\right],
\label{eq:mse-lookahead-exact}
\\
&\mathbb E\!\left[(W_T-W_t^{\mathrm{prf}})^2\right]
= \mathbb E\!\left[\mathrm{Var}(W_T\mid \mathcal F_t)\right]
+ \mathbb E\!\left[(W_t^{\mathrm{prf}})^2(L_t-1)^2\right],
\label{eq:mse-prefix-decomp}
\\
&\mathbb E\!\left[(W_T-W_t^{\mathrm{app}})^2\right]
= \mathbb E\!\left[\mathrm{Var}(W_T\mid \mathcal F_t)\right]
+\mathbb E\!\left[(W_t^{\mathrm{prf}})^2(\widehat L_t-L_t)^2\right]
\label{eq:mse-app-unbiased} \\
& \qquad \qquad \qquad \qquad \leq \mathbb E\!\left[\mathrm{Var}(W_T\mid \mathcal F_t)\right]
+\frac{1}{J}\,\mathbb E\!\left[(W_t^{\mathrm{prf}})^2 C_t\right].
\end{align}

\end{theorem}

%------------------------------------------------------

\begin{proof}
Since
\[
W_T=W_t^{\mathrm{prf}}
\prod_{s=t+1}^T
\frac{m_s(X_s\mid q,X_{<s})}{r_s(X_s\mid q,X_{<s})}\,
\psi_s(X_{1:s},q),
\]
taking the conditional expectation over the future proposal
\(X_{t+1:T}\sim \prod_{s=t+1}^T r_s(\cdot\mid q,X_{<s})\) yields
\[
\mathbb E[W_T\mid \mathcal F_t]
=
W_t^{\mathrm{prf}}\,L_t
=
W_t^{\mathrm{look}}.
\]
Therefore
\[
\mathbb E[(W_T-W_t^{\mathrm{look}})^2]
=
\mathbb E\!\left[\mathrm{Var}(W_T\mid \mathcal F_t)\right],
\]
which proves \eqref{eq:mse-lookahead-exact}.

Next, since \(W_t^{\mathrm{look}}=\mathbb E[W_T\mid \mathcal F_t]\), the error
\(W_T-W_t^{\mathrm{look}}\) is orthogonal to every \(\mathcal F_t\)-measurable random variable,
in particular to \(W_t^{\mathrm{look}}-W_t^{\mathrm{prf}}\). Hence
\[
\mathbb E[(W_T-W_t^{\mathrm{prf}})^2]
=
\mathbb E[(W_T-W_t^{\mathrm{look}})^2]
+
\mathbb E[(W_t^{\mathrm{look}}-W_t^{\mathrm{prf}})^2].
\]
Using \(W_t^{\mathrm{look}}-W_t^{\mathrm{prf}}=W_t^{\mathrm{prf}}(L_t-1)\) gives
\eqref{eq:mse-prefix-decomp}.

Since \(\mathbb E[\widehat L_t\mid \mathcal F_t]=L_t\), then
\(W_t^{\mathrm{app}}=W_t^{\mathrm{prf}}\widehat L_t\) satisfies
\(\mathbb E[W_t^{\mathrm{app}}\mid \mathcal F_t]=W_t^{\mathrm{look}}\), so the same
orthogonality argument gives
\[
\mathbb E[(W_T-W_t^{\mathrm{app}})^2]
=
\mathbb E[(W_T-W_t^{\mathrm{look}})^2]
+
\mathbb E[(W_t^{\mathrm{look}}-W_t^{\mathrm{app}})^2].
\]
Finally,
$W_t^{\mathrm{look}}-W_t^{\mathrm{app}}
=
W_t^{\mathrm{prf}}(L_t-\widehat L_t),$
which yields \eqref{eq:mse-app-unbiased}.

Since $W_t^{\mathrm{prf}}$ is $\mathcal F_t$-measurable,
\begin{align*}
\mathbb E\!\left[(W_t^{\mathrm{prf}})^2(\widehat L_t - L_t)^2\right]
=
\mathbb E\!\left[
(W_t^{\mathrm{prf}})^2
\,\mathbb E\!\left[(\widehat L_t - L_t)^2 \mid \mathcal F_t\right]
\right]
\le
\frac{1}{J}\,
\mathbb E\!\left[(W_t^{\mathrm{prf}})^2 C_t\right].
\end{align*}
Combining this with \eqref{eq:mse-app-unbiased} yields
\[
\mathbb E\!\left[(W_T - W_t^{\mathrm{app}})^2\right]
\le
\mathbb E\!\left[\operatorname{Var}(W_T \mid \mathcal F_t)\right]
+
\frac{1}{J}\,
\mathbb E\!\left[(W_t^{\mathrm{prf}})^2 C_t\right].
\]

\end{proof}

\subsection{Block-wise SMC} \label{section:blockwise-smc}

We now extend the token-wise SMC construction to the case where proposals are made block by block, using blocks of identical size \(B\). The token-wise SMC formulas in the main paper are recovered by setting \(B=1\), so that each block contains a single token. 

Assume for simplicity that \(T=KB\) for some integer \(K\), and partition the sequence as
\[
x_{1:T}
=
\bigl(x_{1:B},\,x_{B+1:2B},\,\dots,\,x_{(K-1)B+1:KB}\bigr).
\]
Thus, block \(k\) is the segment $x_{(k-1)B+1:kB}$ for $k=1,\dots,K.$ The unified full-sequence target is still the same as in the token-wise formulation $\Pi$ defined in \eqref{eq:unified-full-target}. For each block \(k\), define the block transition factor
\begin{equation}
M_k(x_{(k-1)B+1:kB}\mid q,x_{1:(k-1)B})
:=
\prod_{t=(k-1)B+1}^{kB}
m_t(x_t\mid q,x_{<t}),
\label{eq:block:def:Mk}
\end{equation}
and the block reward potential
\begin{equation}
\Psi_k(x_{1:kB},q)
:=
\prod_{t=(k-1)B+1}^{kB}
\psi_t(x_{1:t},q).
\label{eq:block:def:Psik}
\end{equation}
Then $\Pi$ can be rewritten as
\begin{equation}
\Pi(x_{1:T}\mid q)
\propto
\prod_{k=1}^{K}
M_k(x_{(k-1)B+1:kB}\mid q,x_{1:(k-1)B})
\prod_{k=1}^{K}
\Psi_k(x_{1:kB},q).
\label{eq:block:full:target:grouped}
\end{equation}

\begin{lemma}[Prefix-only block-wise SMC target and weight]
Define the unnormalized block-prefix targets
\begin{equation} 
\gamma_k^{\mathrm{prf}}(x_{1:kB}\mid q;\Pi)
=
\prod_{j=1}^{k}
M_j(x_{(j-1)B+1:jB}\mid q,x_{1:(j-1)B})
\prod_{j=1}^{k}
\Psi_j(x_{1:jB},q),
\qquad k=1,\dots,K.
\label{eq:block:simple:gamma}
\end{equation}
Then the following recursion holds:
\begin{equation}
\gamma^{\mathrm{prf}}_k(x_{1:kB}\mid q;\Pi)
=
\gamma^{\mathrm{prf}}_{k-1}(x_{1:(k-1)B}\mid q;\Pi)\,
M_k(x_{(k-1)B+1:kB}\mid q,x_{1:(k-1)B})\,
\Psi_k(x_{1:kB},q).
\label{eq:block:simple:gamma:recursion}
\end{equation}
Let
\begin{equation}
R_k(x_{(k-1)B+1:kB}\mid q,x_{1:(k-1)B})
\label{eq:block:simple:proposal}
\end{equation}
be the proposal distribution used to sample block \(k\). Then the incremental importance weight is
\begin{equation}
w^{\mathrm{prf}}_k(x_{1:kB};\Pi)
=
\frac{
M_k(x_{(k-1)B+1:kB}\mid q,x_{1:(k-1)B})
}{
R_k(x_{(k-1)B+1:kB}\mid q,x_{1:(k-1)B})
}
\Psi_k(x_{1:kB},q).
\label{eq:block:simple:weight}
\end{equation}
\end{lemma}

The proof follows from the definitions of $\gamma^{\mathrm{prf}}$ and $w^{\mathrm{prf}}$ so we omit it here.

% If the proposal factorizes autoregressively within the block,
% \begin{equation}
% R_k(x_{(k-1)B+1:kB}\mid q,x_{1:(k-1)B})
% =
% \prod_{t=(k-1)B+1}^{kB}
% r_t(x_t\mid q,x_{<t}),
% \label{eq:block:proposal:factorized}
% \end{equation}
% then \eqref{eq:block:simple:weight} becomes
% \begin{equation}
% w_k^{\mathrm{prf}}(x_{1:kB};\Pi)
% =
% \left[
% \prod_{t=(k-1)B+1}^{kB}
% \frac{m_t(x_t\mid q,x_{<t})}{r_t(x_t\mid q,x_{<t})}
% \right]
% \left[
% \prod_{t=(k-1)B+1}^{kB}
% \psi_t(x_{1:t},q)
% \right].
% \label{eq:block:simple:weight:factorized}
% \end{equation}

%%%%%%%%%%%%%%%%%%%%%%%%%%%%%%%%%%%%%%%%%%%%%
%%%%%%%%%%%%%%%%%%%%%%%%%%%%%%%%%%%%%%%%%%%%%
%%%%%%%%%%%%%%%%%%%%%%%%%%%%%%%%%%%%%%%%%%%%%
%%%%%%%%%%%%%%%%%%%%%%%%%%%%%%%%%%%%%%%%%%%%%

\begin{lemma}[Lookahead block-wise SMC target and weight]
Define the unnormalized block-prefix marginal targets
\begin{equation}
\gamma_k^{\mathrm{look}}(x_{1:kB}\mid q;\Pi)
=\sum_{x_{kB+1:T}}
\prod_{t=1}^{T} m_t(x_t\mid q,x_{<t})
\prod_{t=1}^{T} \psi_t(x_{1:t},q),
\qquad k=1,\dots,K.
\label{eq:block:lookahead:gamma}
\end{equation}
Further define the block lookahead term
\begin{equation}
L_k(x_{1:kB},q;\Pi)
:=
\sum_{x_{kB+1:T}}
\prod_{t=kB+1}^{T}
m_t(x_t\mid q,x_{<t})
\prod_{t=kB+1}^{T}
\psi_t(x_{1:t},q).
\label{eq:block:lookahead:Lk}
\end{equation}
Then
\begin{equation}
\gamma_k^{\mathrm{look}}(x_{1:kB}\mid q;\Pi)
=
\left[
\prod_{j=1}^{k}
\left(M_j(x_{(j-1)B+1:jB}\mid q,x_{1:(j-1)B})
\Psi_j(x_{1:jB},q) \right)
\right]
L_k(x_{1:kB},q;\Pi),
\label{eq:block:lookahead:gamma:factorized}
\end{equation}
and therefore
\begin{equation}
\frac{\gamma_k^{\mathrm{look}}(x_{1:kB}\mid q;\Pi)}{\gamma_{k-1}^{\mathrm{look}}(x_{1:(k-1)B}\mid q)}
=
M_k(x_{(k-1)B+1:kB}\mid q,x_{1:(k-1)B})\,
\Psi_k(x_{1:kB},q)\,
\frac{L_k(x_{1:kB},q;\Pi)}{L_{k-1}(x_{1:(k-1)B},q;\Pi)}.
\label{eq:block:lookahead:gamma:recursion}
\end{equation}
Hence, given proposal
\(R_k(x_{(k-1)B+1:kB}\mid q,x_{1:(k-1)B})\),
the incremental importance weight is
\begin{equation}
w_k^{\mathrm{look}}(x_{1:kB};\Pi)
=
\frac{
M_k(x_{(k-1)B+1:kB}\mid q,x_{1:(k-1)B})
}{
R_k(x_{(k-1)B+1:kB}\mid q,x_{1:(k-1)B})
}
\Psi_k(x_{1:kB},q)
\frac{L_k(x_{1:kB},q;\Pi)}{L_{k-1}(x_{1:(k-1)B},q;\Pi)}.
\label{eq:block:lookahead:weight}
\end{equation}
\end{lemma}

The proof follows from the definitions of $\gamma^{\mathrm{look}}$ and $w^{\mathrm{look}}$ so we omit it here.

% If the proposal factorizes autoregressively within the block as in
% \eqref{eq:block:proposal:factorized}, then \eqref{eq:block:lookahead:weight} becomes
% \begin{equation}
% w_k(x_{1:kB})
% =
% \left[
% \prod_{t=(k-1)B+1}^{kB}
% \frac{m_t(x_t\mid q,x_{<t})}{r_t(x_t\mid q,x_{<t})}
% \right]
% \left[
% \prod_{t=(k-1)B+1}^{kB}
% \psi_t(x_{1:t},q)
% \right]
% \frac{L_k(x_{1:kB},q)}{L_{k-1}(x_{1:(k-1)B},q)}.
% \label{eq:block:lookahead:weight:factorized}
% \end{equation}

\subsection{Block-wise Metropolis--Hastings} \label{sec:blockwise-mh}

\paragraph{Block-wise Metropolis--Hastings moves.}
At stage \(k\), after proposing or resampling a block, one may apply a block-wise
Metropolis--Hastings (MH) move to rejuvenate particles while preserving the corresponding
intermediate target. The MH move keeps the prefix \(x_{1:(k-1)B}\) fixed and updates only the
current block. This is the block analogue of using MH moves inside SMC for the token-wise targets defined in the paper. 

Let $z := x_{(k-1)B+1:kB}$ denote the current block, and suppose we propose
\[
z' \sim Q_k(\,\cdot \mid q, x_{1:(k-1)B}, z),
\]
where \(Q_k\) is a proposal kernel on the block. Define
\[
x_{1:kB} = (x_{1:(k-1)B}, z),
\qquad
x'_{1:kB} = (x_{1:(k-1)B}, z').
\]

\begin{lemma}[Block-wise MH move for the prefix-only intermediate target]
Let \(\gamma_k^{\mathrm{prf}}(x_{1:kB}\mid q;\Pi)\) be the prefix-only block-wise intermediate target defined in \eqref{eq:block:simple:gamma}. 
Then a Metropolis--Hastings update of block \(k\) that preserves this target has acceptance
probability
\begin{equation}
a_k^{\mathrm{prf}}(z,z')
=
\min\!\left\{
1,\;
\frac{
\gamma_k^{\mathrm{prf}}(x'_{1:kB}\mid q;\Pi)\,
Q_k(z \mid  q, x_{1:(k-1)B}, z')
}{
\gamma_k^{\mathrm{prf}}(x_{1:kB}\mid q;\Pi)\,
Q_k(z' \mid  q, x_{1:(k-1)B}, z)
}
\right\}.
\label{eq:block:mh:simple:generic}
\end{equation}
Using the factorization of \(\gamma_k^{\mathrm{prf}}\), this can be written as
\begin{equation}
a_k^{\mathrm{prf}}(z,z')
=
\min\!\left\{
1,\;
\frac{
M_k(z' \mid q,x_{1:(k-1)B})\,
\Psi_k(x'_{1:kB},q)\,
Q_k(z \mid q, x_{1:(k-1)B}, z')
}{
M_k(z \mid q,x_{1:(k-1)B})\,
\Psi_k(x_{1:kB},q)\,
Q_k(z' \mid  q, x_{1:(k-1)B}, z)
}
\right\}.
\label{eq:block:mh:simple}
\end{equation}
\end{lemma}

\begin{lemma}[Block-wise MH move for the lookahead intermediate target]
Let \(\gamma_k^{\mathrm{look}}(x_{1:kB}\mid q;\Pi)\) be the lookahead block-wise intermediate target defined in \eqref{eq:block:lookahead:gamma}.
Then a Metropolis--Hastings update of block \(k\) that preserves this target has acceptance
probability
\begin{equation}
a_k^{\mathrm{look}}(z,z')
=
\min\!\left\{
1,\;
\frac{
\gamma_k^{\mathrm{look}}(x'_{1:kB}\mid q)\,
Q_k(z \mid  q, x_{1:(k-1)B}, z')
}{
\gamma_k^{\mathrm{look}}(x_{1:kB}\mid q)\,
Q_k(z' \mid  q, x_{1:(k-1)B},z)
}
\right\}.
\label{eq:block:mh:lookahead:generic}
\end{equation}
Using the factorization of \(\gamma_k^{\mathrm{look}}\), this can be written as
\begin{equation}
a_k^{\mathrm{look}}(z,z')
=
\min\!\left\{
1,\;
\frac{
M_k(z' \mid q,x_{1:(k-1)B})\,
\Psi_k(x'_{1:kB},q)\,
L_k(x'_{1:kB},q)\,
Q_k(z \mid  q, x_{1:(k-1)B}, z')
}{
M_k(z \mid q,x_{1:(k-1)B})\,
\Psi_k(x_{1:kB},q)\,
L_k(x_{1:kB},q)\,
Q_k(z' \mid  q, x_{1:(k-1)B}, z)
}
\right\}.
\label{eq:block:mh:lookahead}
\end{equation}
\label{lemma:targetII:mh}
\end{lemma}

The proofs of the above lemmas follow directly from the definition of the Metropolis–Hastings acceptance ratio and are therefore omitted.

\end{document}